\documentclass[preprint,12pt]{elsarticle}

\usepackage{amssymb}
\usepackage{amsmath}
\usepackage{graphicx}
\usepackage{bm}
\usepackage{booktabs}
\usepackage{multirow}
\usepackage{array}
\usepackage{color}
\usepackage{hyperref}

\usepackage{algorithm}
\usepackage{algorithmic}
\usepackage{amsmath}
\usepackage{amsthm}
\usepackage{booktabs}
\usepackage{algorithm}
\usepackage{algorithmic}
\usepackage[dvipsnames]{xcolor}
\usepackage{array}
\definecolor{mc0}{HTML}{66FFFF}

\newcommand{\OP}[1]{\mathop{\mathrm{#1}}}
\newcolumntype{C}[1]{>{\centering\arraybackslash}m{#1}}
\newcolumntype{L}[1]{>{\arraybackslash}m{#1}}

\newtheorem{lemma}{Lemma}
\newtheorem{proposition}{Proposition}

\newtheorem{corollary}{Corollary}

\journal{Arxiv}

\begin{document}

\begin{frontmatter}



\title{PCaM: A Progressive Focus Attention-Based Information Fusion Method for Improving Vision Transformer Domain Adaptation}



\author[1,2]{Zelin Zang}
\author[5]{Fei Wang} 
\author[1]{Liangyu Li} 
\author[2]{Jinlin Wu} 
\author[7]{Chunshui Zhao} 
\author[2,3,4,6]{Zhen Lei$^\dagger$}
\author[5]{Baigui Sun$^\dagger$} 

\affiliation[1]{Westlake Institute for Advanced Studies, Westlake University; HangZhou; China}
\affiliation[2]{Centre for Artificial Intelligence and Robotics (CAIR); HKISI-CAS}
\affiliation[3]{State Key Laboratory of Multimodal Artificial Intelligence Systems (MAIS); Institute of Automation; Chinese Academy of Sciences (CASIA)}
\affiliation[4]{School of Artificial Intelligence; University of Chinese Academy of Sciences (UCAS)}
\affiliation[5]{Alibaba Group; HangZhou; China}
\affiliation[6]{School of Computer Science and Engineering, the Faculty of Innovation Engineering, Macau University of Science and Technology, Macau 999078, China }
\affiliation[7]{Microsoft}

\begin{abstract}
    Unsupervised Domain Adaptation (UDA) aims to transfer knowledge from a labeled source domain to an unlabeled target domain. Recent UDA methods based on Vision Transformers (ViTs) have achieved strong performance through attention-based feature alignment. However, we identify a key limitation: foreground object mismatch, where the discrepancy in foreground object size and spatial distribution across domains weakens attention consistency and hampers effective domain alignment. To address this issue, we propose the Progressive Focus Cross-Attention Mechanism (PCaM), which progressively filters out background information during cross-attention, allowing the model to focus on and fuse discriminative foreground semantics across domains. We further introduce an attentional guidance loss that explicitly directs attention toward task-relevant regions, enhancing cross-domain attention consistency. PCaM is lightweight, architecture-agnostic, and easy to integrate into existing ViT-based UDA pipelines. Extensive experiments on Office-Home, DomainNet, VisDA-2017, and remote sensing datasets demonstrate that PCaM significantly improves adaptation performance and achieves new state-of-the-art results, validating the effectiveness of attention-guided foreground fusion for domain adaptation.
\end{abstract}




\end{frontmatter}


\section{Introduction}
\label{sec:intro}

\begin{figure*}[t]
    \centering
    \includegraphics[width=0.99\linewidth]{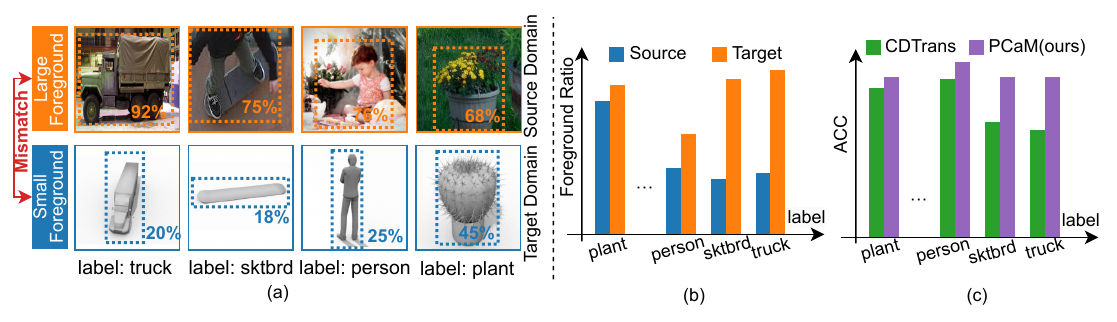}
    \caption{\textbf{Highlighting the Foreground Object Mismatch Issue and the Enhancement Achieved with PCaM in the VisDA Dataset.} Part (a) presents bar plots indicating foreground ratios (defined by the proportion of pixels encircled in red to the total pixel count) for individual categories across both source and target domains. Part (b) depicts the bar plots representing the accuracy of each category. In contrast, part (c) showcases representative samples from each category with the foreground entities emphasized in green/purple.}
    \label{fig_problem}
\end{figure*}

To address the frequent domain shifts that challenge the generalization ability of pre-trained backbone models~\cite{zang2022dlme}, researchers have increasingly focused on unsupervised domain adaptation (UDA) strategies~\cite{LUO2025103197,zang2023boosting,jiang2025multi}. UDA approaches leverage knowledge from labeled source domains to improve model performance in unlabeled target domains, particularly under the framework of visual transformer (ViT)-based UDA methods. Within this context, methods like CDTrans~\cite{xu2021cdtrans} utilize cross-attention mechanisms~\cite{yang2025ccin,li2024crossfuse} to enhance the learning of domain-invariant features by guiding the attention across domains. 

Despite these advancements, our analysis has revealed a critical limitation in ViT-based transfer learning models: the \textit{Foreground Object Mismatch Issue (\textbf{FOM Issue})}. This issue arises when there is a substantial size disparity between the foreground object in the source and target domains. We argue that such discrepancies hinder the effective integration of domain features through cross-attention, as the mechanism becomes susceptible to noise during feature alignment. Specifically, mismatched foreground sizes distort the model's capacity to accurately align features, thus impairing performance.
The FOM Issue is especially pronounced in datasets like VisDA. Our analysis of the dataset (Fig.\ref{fig_problem}) demonstrates a negative correlation between the degree of foreground proportion mismatch between source and target domains (Fig.\ref{fig_problem}b) and the accuracy of predictions for corresponding categories (Fig.~\ref{fig_problem}c). This evidence suggests that the FOM Issue poses a significant obstacle to the performance of baseline methods in UDA, indicating the need for more nuanced strategies to address foreground mismatches for robust domain transfer.

The FOM issue has been tangentially discussed in other domains, including image processing~\cite{zhao2025multi}. Some concrete solutions~\cite{liu2021f2net,lin2023foreground} enhance model recognition performance by directing the model's focus towards foreground objects~(more related words are in the appendix). Regrettably, the transfer learning paradigm often overlooks the deleterious effects of the FOM Issue, concentrating solely on feature alignment between source and target domains.

To mitigate the \textit{Foreground Object Mismatch Issue (FOM Issue)}, we introduce the \textbf{P}rogressively-Focused \textbf{C}ross-\textbf{a}ttention \textbf{M}echanism, or \textbf{PCaM} (illustrated in Fig.~\ref{fig_intro}). PCaM is designed to tackle foreground-background discrepancies by selectively enhancing foreground focus and minimizing background interference. This approach consists of a modular, plug-and-play transformer component paired with a concentration loss function that jointly enable precise `attention focus' on relevant regions. Specifically, PCaM identifies critical local regions through a progressive attention refinement process, which involves feature cropping to hone in on foreground elements. The concentration loss further stabilizes this attention during fine-tuning by regulating attention map variance, thus preventing noisy alignments and promoting seamless domain adaptation. As seen in Fig.~\ref{fig_problem}c, PCaM achieves substantial performance gains in challenging categories, underscoring its effectiveness in addressing the FOM Issue.
Our contributions are as follows:
(a) \textbf{Identification of the FOM Issue}: We formally identify the FOM Issue in ViT-based UDA, illustrating how addressing this issue can significantly improve cross-domain performance.
(b) \textbf{Introduction of PCaM}: We propose a novel approach, PCaM, to counteract the FOM Issue through attention rollout, feature refinement, and progressively focused loss, facilitating precise foreground extraction.
(c) \textbf{Empirical Validation}: Extensive experiments validate the effectiveness of PCaM across several UDA benchmarks, including VisDA-2017, DomainNet, remote sensing, where it achieves SOTA results, demonstrating its robustness and scalability across domains.

\section{Related Work}

\textbf{Unsupervised Domain Adaptation.}
Unsupervised Domain Adaptation (UDA) tackles the challenge of transferring knowledge to a target domain lacking labeled data by leveraging innovative methods to bridge the domain gap. FixBi~\cite{Jaemin2021FixBi} enhances cross-domain alignment by introducing multiple intermediate domains, while CDTrans~\cite{xu2021cdtrans} demonstrates robustness against noisy labels and excels in feature alignment. DOT~\cite{DOT_2022_mm} employs a transformer-based approach, using categorical tokens from both source and target domains to effectively capture domain-invariant and domain-specific features. PMTrans~\cite{PTrans2023} addresses the domain gap through game-theoretic strategies, connecting source and target domains via intermediary domains. Similarly, SAN~\cite{zang2023boosting} leverages unsupervised contrastive learning to handle open-set and partial-set domain adaptation challenges. Despite these advancements, existing UDA methods often falter in scenarios involving foreground object mismatches, which can lead to unreliable and error-prone outcomes.

\textbf{Cross-Attention Mechanism.}
The attention mechanism plays a critical role in sequence modeling and transduction tasks, demonstrating remarkable success across diverse applications such as abstractive summarization, reading comprehension, textual entailment, and learning task-independent sentence representations~\cite{chen2025self}. Building on this foundation, cross-attention modules extend the capability of attention by processing inputs from two distinct modalities, effectively aligning and aggregating information between them. This approach has proven invaluable in tasks like vision-to-text~\cite{Yao_2019_ACL}, and vision-to-vision~\cite{Li_2021_CVPR}, enabling more effective integration and understanding of multimodal data.

\section{Methodology}
\label{sec:method}

\textbf{Notations. } This work follows the Unsupervised Domain Adaptation (UDA) framework, utilizing labeled source domain data, $\mathcal{X}^{s} = \{\bm{x}^{s}_i\}_{i=1}^{N^{s}}, \mathcal{Y}^{s} = \{\bm{y}^{s}_i\}_{i=1}^{N^{s}}$, and unlabeled target domain data, $\mathcal{X}^{t} = \{\bm{x}^{t}_i\}_{i=1}^{N^{t}}$. Both domains share the same label space but have distinct distributions.

\begin{figure*}[t]
   \centering
   \includegraphics[width=0.99\linewidth]{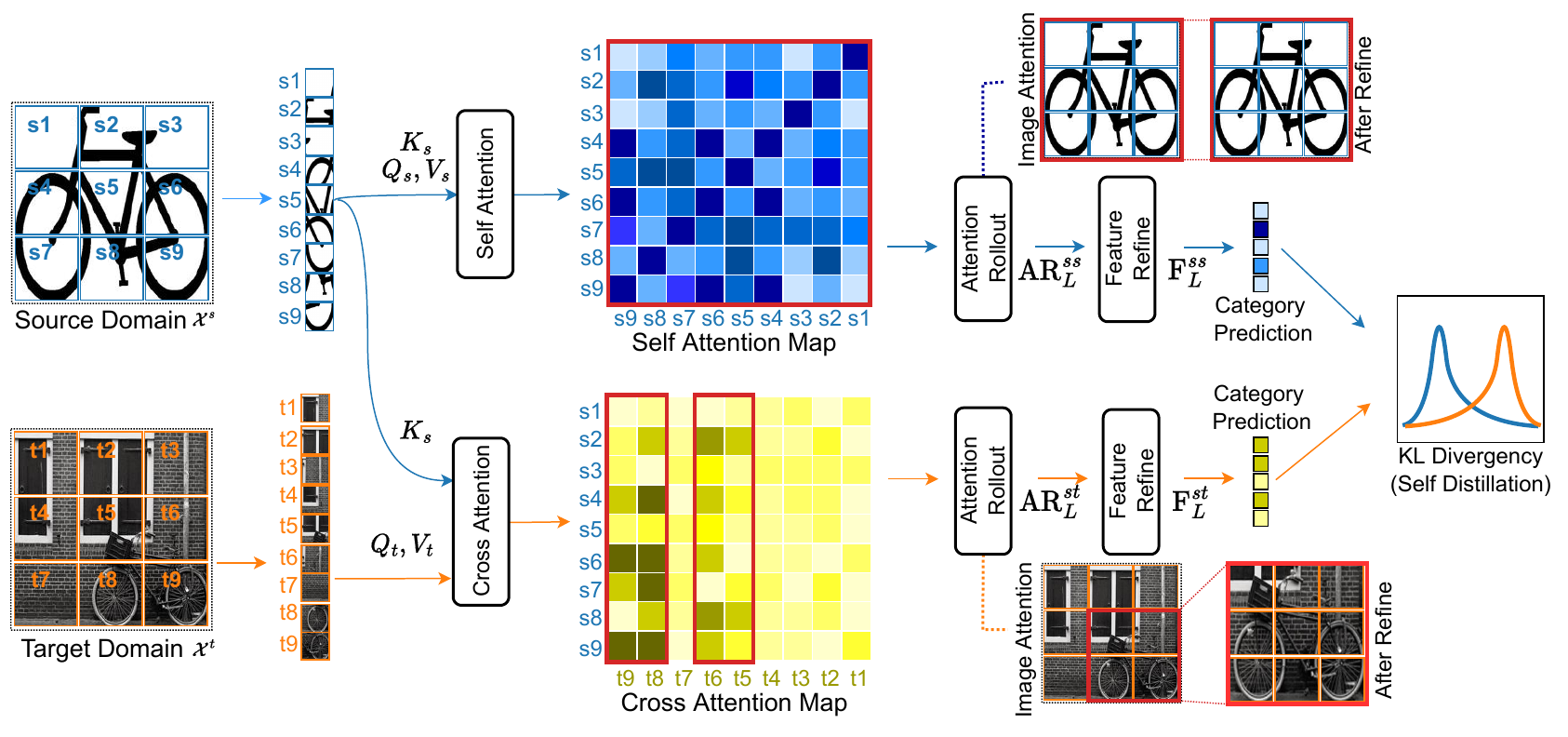}
   \caption{\textbf{Core concept of PCaM.} PCaM addresses the FOM issue by utilizing the cross-attention module to isolate the foreground and filter out unrelated backgrounds. As illustrated, bicycle samples from two distinct domains are harmonized to comparable foreground scales using PCaM, facilitating improved transfer learning. Within the attention map, a darker hue indicates heightened attention.}
   \label{fig_intro}
\end{figure*}

\vspace{2mm}
\subsection{Preliminaries: Vision Transformer (ViT)}

ViT has successfully adapted the attention-based architecture from NLP to computer vision tasks~\cite{dosoviTskiy2020image}. ViT takes an input image $\bm{x}\in \mathbf{R}^{H\times W\times C}$ and divides it into patches. Each patch $\bm{p}_j^{\bm{x}} \in \mathbf{R}^{P\times P\times C}$ is then mapped to a flattened embedding $\bm{z}_j\in \mathbf{R}^{D}$ through a trainable linear projection, where $(P\times P)$ represents the resolution of each patch, and $N = HW/P^2$ represents the total number of patches. To preserve positional information, a $\mathrm{[CLS]}$ token, $\bm{z}_0 = \bm{x}_{\mathrm{[CLS]}}\in \mathbf{R}^{D}$, is added at the beginning of the sequence of patch embeddings. Position embeddings $\bm{E}_{pos}$ are incorporated into each patch embedding. The input states, $\bm{z} = [\bm{z}_0; \bm{z}_1; \dots; \bm{z}_N]$, then traverse through $L$ layers of multi-head self-attention (MSA) and multi-layer perceptron (MLP) alternately. Layer normalization (LN) is applied before each block, and skip connections are employed after each block to facilitate training.

The token embeddings are projected into query, key, and value representations: $\bm{Q}, \bm{K}, \bm{V} \in \mathbf{R}^{N\times D_h}$~\cite{chen2025self}. The following gather computes the output,
\begin{equation}
   \OP{A}(\bm{Q}, \bm{K}, \bm{V}) = \text{softmax}(\bm{Q}\bm{K}^T/\sqrt{D_h}) \bm{V}.
   \label{eqn:attn}
\end{equation}
In self-attention, $\bm{Q}, \bm{K}, \bm{V}$ are computed from the same embeddings. The multi-head self-attention (MSA) is a multi-channel version of attention, defined as follows:
\begin{equation}
   \OP{M}(\bm{z})\!=\!\text{Concat}\left(\!\OP{A}(\bm{Q}_1,\!\bm{K}_1,\!\bm{V}_1)\!,\!\dots\!,\!\OP{A}(\bm{Q}_H,\!\bm{K}_H,\!\bm{V}_H)\!\right)\!,
   \label{eqn:msa}
\end{equation}
where the $\bm{z}^{(0)}$ is input patch embeddings and is defined by the following gathers,
$\bm{z}^{(l)}_{a}  = \bm{z}^{(l-1)} + \OP{M}(\bm{z}^{(l-1)})$,
$\bm{z}^{(l)}      = \bm{z}^{(l)}_{a} + \OP{MLP}(\bm{z}^{(l)}_{a})$
Finally, the $\mathrm{CLS}$ token of the last layer is regarded as the encoded feature, $\bm{f} = \bm{z}^{(L)}_0$.

\vspace{2mm}
\subsection{Data pairing with pseudo labels}

We construct training pairs for the cross-attention module by identifying the best-matching sample from the target domain for each source domain sample. Pseudo-labeled samples are paired from both domains and presented to the model as pairs. For instance, one data point corresponds to the source sample $\mathcal{X}^{s}$, and another corresponds to the target sample $\mathcal{X}^{t}$, resulting in a pair of source and target data $\{\mathcal{X}^{s}, \mathcal{X}^{t}\}$. Pairing within the training dataset relies on similarities in features between images from the two domains. We also employ target samples to match the most similar corresponding samples from the source domain, thereby ensuring sufficient training samples. The set $\mathcal{P}_S $ of selected pairs is defined as follows,
\begin{equation}
   \mathcal{P}_S = \{(\bm{x}^{s}_i, \bm{x}^{t}_j) | i = \arg\min_{i' \in S} d(\bm{f}^{t}_j, \bm{f}^{s}_{i'}), \forall j \in T\},
   \label{eq:ps}
\end{equation}
where $S$ and $T$ represent the source and target data, respectively. $d(\bm{f}^{t}_j,\bm{f}^{s}_i)$ denotes the distance between features of images $i$ and $j$, and $i'$  represents the index of candidate source samples.

For refining target pseudo-labels, all target data is fed into the pre-trained model to obtain the classifier's probability distribution ($\delta$) over source categories. A center-aware filtering strategy is applied to eliminate noisy paired samples from the set. This strategy calculates initial category centers in the target domain using weighted k-means clustering based on these distributions,
\begin{equation}
   \boldsymbol{c}_k=(\sum_{i\in I}\delta_i^k\boldsymbol{f}^t)/(\sum_{i\in I}\delta_i^k),
   y_t=\mathop{\arg\min_k}d(\boldsymbol{c}_k,\boldsymbol{f}^t),
   \label{eq:initcenter}
\end{equation}
where $\delta_i^k$ denotes the probability that image $i$ belongs to category $k$. Pseudo-labels $y_t$ are generated from the target data based on a nearest neighbor classifier, where $i \in I$. Using these pseudo-labels, the centers of the categories in the target domain can be further updated to improve the accuracy of the pseudo-labels. This process can be completed in just one round, and the final pseudo-labels are used to optimize the selected sample pairs, which are retained for training only if the pseudo-labels of the target domain samples are consistent with those of the source domain samples, otherwise they are discarded. In addition, to further ensure the quality of pairs, we introduce a screening mechanism based on similarity threshold, if the feature similarity of the paired samples does not reach the preset threshold, even if the pseudo-labels match, they will be considered as low-quality pairs and removed to minimize the impact of the mismatched samples on the performance of the model.

\begin{figure*}
   \centering
   \includegraphics[width=0.99\linewidth]{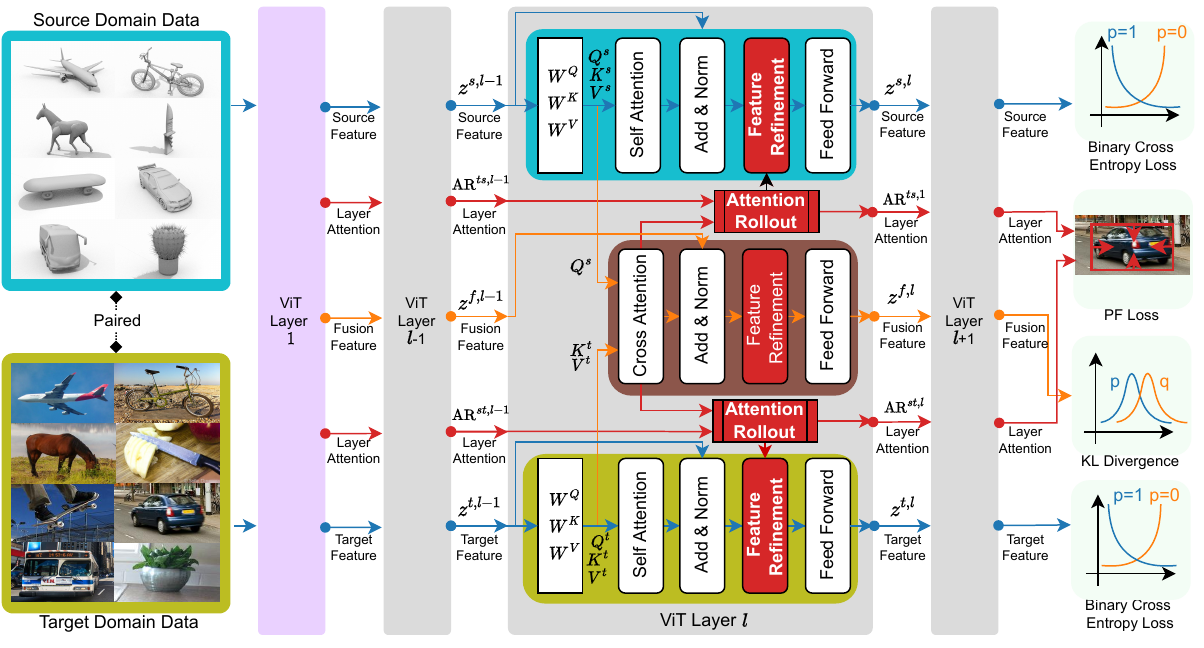}
   \caption{\textbf{Framework of Proposed PCaM.} Our PCaM framework comprises Attention Rollout and PF Loss components. Attention Rollout guides feature cropping, while PF Loss encourages the model to focus on a continuous local region. The source/target feature contain the information of the source/target domain, the fusion feature is the concatenation of the source and target feature.}
   \label{fig_method}
\end{figure*}

\vspace{2mm}
\subsection{The proposed PCaM Framework}

To address the \textit{FOM issue}, as shown in Fig.~\ref{fig_intro} and Fig.~\ref{fig_method}, we introduce PCaM. PCaM seamlessly integrates as a plug-and-play module within the ViT-based UDA framework. PCaM encompasses three essential stages: \textit{attention rollout~(AR)}, \textit{feature refinement~(FR)}, and \textit{progressively focused loss functions~(PF Loss)}. The attention rollout process identifies highly correlated regions in paired images through ViT's cross-attention mechanism, postulating these regions to correspond to foreground objects. The attention refinement process aligns these regions to a consistent scale, facilitating effective feature extraction. Finally, the progressively focused loss functions guide the model in gradually emphasizing the foreground portions of the data. 

\textbf{Attention Rollout~(AR).}
The attention rollout process identifies highly correlated regions in paired images through ViT's cross-attention mechanism, postulating these regions to correspond to foreground objects. We formally prove in \ref{sec:app:attention_rollout} that attention rollout is bounded, convergent, and exhibits consistent foreground aggregation behavior under layer-wise accumulation. It provides an estimation of the relative importance of input tokens based on their corresponding attention weights. In this work, we extend the role of AR to guide the model's gradual focus on the foreground component of the data.
Specifically, AR extracts the regions emphasized by both self-attention and source-target cross-attention. For the source-target cross-attention rollout $\text{AR}_{j}^{st,l}$ of layer $l$, it is defined as $\text{AR}_{j}^{st,l}  = \sum_{i=1}^{N} \bar{\text{AR}}_{i,j}^{st,l}$,
\begin{equation}
   \small
   \begin{aligned}
      \bar{\text{AR}}_{i,j}^{st,l} & \!=\!
      \left\{
      \begin{aligned}
          & \bar{\text{AR}}_{i,j}^{st,l-1} \!+\! \sigma\left((z^{s,l}_{i} \cdot z^{t,l}_{j}) / \sqrt{d} \right) & l>1 \\
          & \sigma\left((z^{s,l}_{i} \cdot z^{t,l}_{j}) / \sqrt{d} \right)                                      & l=1
      \end{aligned}
      \right.
   \end{aligned},
   \label{eq_attentionrollout}
\end{equation}
where $\bar{\text{AR}}_{i,j}^{st,l}$ describes a patch-wise attention map between the $i$-th patch in the source domain and the $j$-th patch in the target domain. The function $\sigma(\cdot)$ denotes a softmax activation, which ensures that attention weights are normalized to a probability distribution. The embedding $z^{s,l}_{i}$ corresponds to the $i$-th patch's embedding in layer $l$ of the source domain, and $z^{t,l}_{j}$ is the embedding of the $j$-th patch in layer $l$ of the target domain, with $i,j \in \{1, \ldots, N\}$. $N$ represents the number of patches within image. 
A recursive process is used to compute the AR for layer $L$, which propagates information back to the initial layer for global aggregation. We analyze this process theoretically in \ref{sec:app:attention_rollout} and show that the attention rollout remains numerically stable and accumulates semantically consistent foreground regions.

\textbf{Feature Refinement~(FR).} FR operation focuses the model's perception on the foreground indicated by AR and includes two operations of box identification and box interpolation.
In the box identification operation, the attention rollout $\text{AR}_{j}^{st,l}$ provides interconnectedness information. The highlighted region~(with a high value of $\text{AR}_{j}^{st,l}$) indicates a high correlation between the $i$-th patch in the source domain and the $j$-th domain in the target domain.
Upsampling the attention map to the original image size helps locate the regions identified by the model as highly relevant for the given class, approximating the object's location in the image. The box identification operation localizes the object during training by utilizing the information in the attention map,
\begin{equation}
   [a_{<}^{s,l},a_{>}^{s,l},a_{\wedge}^{s,l},a_{\vee}^{s,l}] =\text{BI}(\text{AR}_{j}^{st,l}, \beta),
   \label{eqn_attn}
\end{equation}
where the scale of the bounding box is controlled by the threshold parameter $\beta \in [0, 1]$. The $a_{<}^{s,l}$, $a_{>}^{s,l}$, $a_{\wedge}^{s,l}$, $a_{\vee}^{s,l}$ are the left, right, high and low boundaries of the bounding box. The function $BI$ is defined to find a box of the attention map defined by $[a_{<}^{s,l}, a_{>}^{s,l}, a_{\wedge}^{s,l}, a_{\vee}^{s,l}]$, which is defined as,
\begin{equation}
   \small
   \left\{
   \begin{aligned}
      \!a_{<}^{s,l}      & \!=\! \arg \min \{\! I_R(j)\ |\ \text{AR}_{j}^{st,l} > \beta, j \in \{1, \ldots, N\} \!\}\! \\
      \!a_{>}^{s,l}      & \!=\! \arg \max \{\! I_R(j)\ |\ \text{AR}_{j}^{st,l} > \beta, j \in \{1, \ldots, N\} \!\}\! \\
      \!a_{\vee}^{s,l}   & \!=\! \arg \min \{\! I_C(j)\ |\ \text{AR}_{j}^{st,l} > \beta, j \in \{1, \ldots, N\} \!\}\! \\
      \!a_{\wedge}^{s,l} & \!=\! \arg \max \{\! I_C(j)\ |\ \text{AR}_{j}^{st,l} > \beta, j \in \{1, \ldots, N\}\!\}\!  \\
   \end{aligned}
   \right.
   \label{eqn_attn_2}
\end{equation}
where $I_R(j)$ and $I_C(j)$ map the patch index $j$ to the row and column index in the original image. Then, the average foreground rate $\Omega=\sum (a_{>}^{s,l}-a_{<}^{s,l})(a_{\wedge}^{s,l}-a_{\vee}^{s,l})/N^{p}$, where $N^p$ is the number of pixels of a single image. The stability of the box extraction process with respect to attention thresholding and attention noise is analyzed in \ref{sec:app:box_identification}. We show that the identified region varies piecewise-continuously with respect to threshold $\beta$ and is robust to small perturbations.

\begin{figure*}[t]
   \centering
   \includegraphics[width=0.99\linewidth]{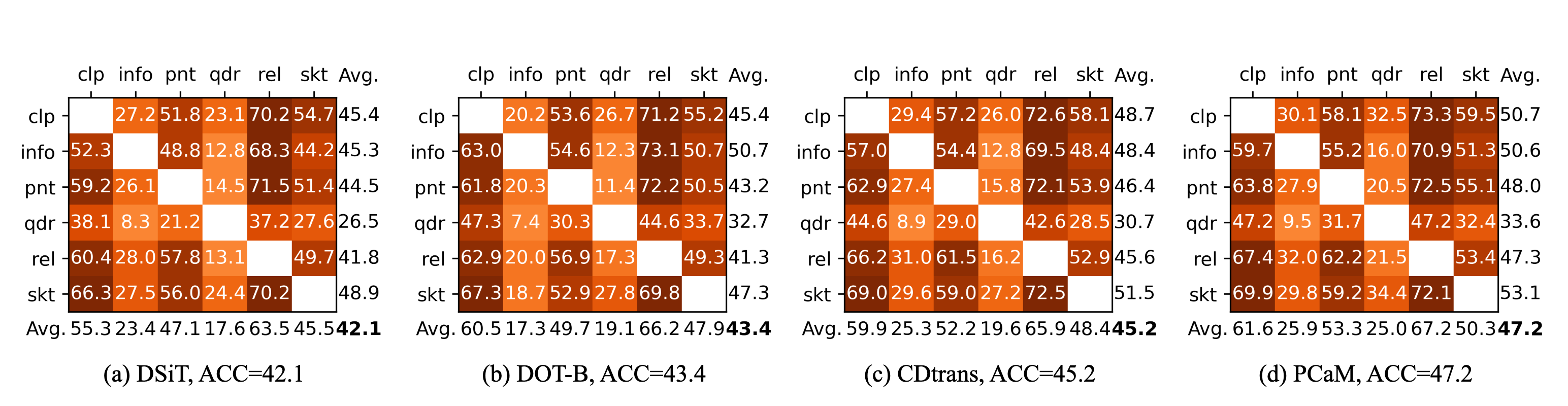}
   \vspace{-4mm}
   \caption{
      \textbf{Confusion matrix on DomainNet.} Comparison on confusion matrix with DsiT, DOT, and CDTrains on DomainNet. PCaM achieves the best performance on the DomainNet dataset. PCaM outperforms the second best method by \textbf{2.0\%}.
   }
   \label{fig:domainet}
\end{figure*}
\begin{table*}[t]
	\centering
	\caption{\textbf{Comparison of our approach with SOTA methods on Office-Home.} The highest acc is highlighted in \textbf{bold}. BL indicates the baseline of PCaM. PCaM outperforms the second best method by \textbf{0.7\%} outperforms the BL by \textbf{0.8\%}.}
	\label{tab:commands}
	{
		\footnotesize
		\begin{tabular}{l|cccccc|c}
			\toprule
			\textbf{Metric} & \textbf{CGDM} & \textbf{DOT}  & \textbf{DSiT} & \textbf{CDTrans} & \textbf{DPP\&BST} & \textbf{C-SFTrans} & \textbf{PCaM}   \\
			Year            & 2021          & 2022          & 2023          & 2023             & 2024              & 2024               & Ours            \\
			\midrule
			Ar2Cl           & 67.1          & 69.0          & 69.2          & 68.8             & 69.8              & 70.3               & \textbf{71.2}   \\
			Ar2Pr           & 83.9          & 85.6          & 83.5          & 85.0             & 84.7              & 83.9               & \textbf{86.2}   \\
			Ar2Re           & 85.4          & 87.0          & 87.3          & 86.9             & 87.0              & 87.3               & \textbf{87.6}   \\
			Cl2Ar           & 77.2          & 80.0          & 80.7          & 81.5             & 80.5              & 80.2               & \textbf{82.0}   \\
			Cl2Pr           & 83.3          & 85.2          & 86.1          & 87.1             & 86.8              & 86.9               & \textbf{87.5}   \\
			Cl2Re           & 83.7          & 86.4          & 86.2          & 87.3             & 86.5              & 86.1               & \textbf{87.6}   \\
			Pr2Ar           & 74.6          & 78.2          & 77.9          & 79.6             & 78.3              & 78.9               & \textbf{79.8}   \\
			Pr2Cl           & 64.7          & \textbf{65.4} & 67.9          & 63.3             & 65.1              & 65.0               & 65.3            \\
			Pr2Re           & 85.6          & 87.9          & 86.6          & 88.2             & 88.1              & 87.7               & \textbf{88.4}   \\
			Re2Ar           & 79.3          & 79.7          & \textbf{82.4} & 82.0             & 81.8              & \textbf{82.6}      & \textbf{82.4}   \\
			Re2Cl           & \textbf{69.5} & 67.3          & 68.3          & 66.0             & 66.9              & 67.9               & 66.3            \\
			Re2Pr           & 87.7          & 89.3          & 89.8          & 90.6             & 90.1              & 90.2               & \textbf{90.7}   \\
			\midrule
			AVE             & 78.5          & 80.1          & 80.5          & 80.5             & 80.6              & 80.6               & \textbf{81.3}   \\
			$\Delta$        &               &               &               &                  &                   &                    & \textbf{(+0.7)} \\
			\bottomrule
		\end{tabular}
	}
	\label{tab:officehome}
\end{table*}


\begin{table*}[h]
	\centering
	\caption{\textbf{Comparison of our approach with SOTA methods on VisDA-2017.} PCaM outperforms the second best method by \textbf{1.1\%} and outperforms BL by \textbf{3.0\%}. 
		The highest acc is highlighted in \textbf{bold}.}
	\footnotesize
	{
		\label{tab_abs_visda}
		\begin{tabular}{l|cccccc|c}
			\toprule
			\textbf{Metric} & \textbf{CGDM} & \textbf{DOT}  & \textbf{DSiT} & \textbf{CDTrans} & \textbf{DPP\&BST} & \textbf{C-SFTrans} & \textbf{PCaM}   \\
			Year            & 2021          & 2022          & 2023          & 2023             & 2024              & 2024               & Ours            \\
			\midrule
			plane           & 96.0          & \textbf{99.3} & 98.2          & 97.7             & 98.0              & 98.6               & 98.1            \\
			bcycl           & 87.1          & {92.7}        & 90.9          & 86.4             & 90.2              & 91.5               & \textbf{93.2}   \\
			bus             & 86.8          & 89.0          & 89.1          & 86.9             & 87.5              & 88.5               & \textbf{90.1}   \\
			car             & \textbf{86.8} & 78.8          & 82.0          & 83.3             & 84.0              & 84.5               & 89.4            \\
			horse           & 92.2          & 98.2          & 97.9          & 97.8             & 97.9              & 98.4               & \textbf{98.8}   \\
			knife           & \textbf{98.3} & 96.1          & 97.4          & 97.2             & 97.6              & 97.9               & 97.3            \\
			mcycl           & 91.6          & 93.1          & 94.8          & 95.9             & 95.3              & 95.1               & \textbf{96.0}   \\
			person          & 78.5          & 80.2          & 81.9          & {84.1}           & 83.4              & 83.2               & \textbf{84.6}   \\
			plant           & 96.3          & 97.6          & 97.7          & 97.9             & 97.2              & 97.5               & \textbf{98.0}   \\
			sktbrd          & 48.4          & \textbf{95.8} & 91.2          & 83.5             & 90.7              & 91.7               & 92.1            \\
			train           & 89.4          & 94.4          & 94.7          & 94.6             & 94.8              & 95.0               & \textbf{95.4}   \\
			truck           & 39.0          & \textbf{69.0} & 60.0          & 55.3             & 61.2              & 63.5               & 63.8            \\
			\midrule
			AVE             & 82.3          & 90.3          & 88.5          & 88.4             & 89.0              & 89.6               & \textbf{91.4}   \\
			$\Delta$        &               &               &               &                  &                   &                    & \textbf{(+3.0)} \\
			\bottomrule
		\end{tabular}
	}
	\label{tab:visda}
\end{table*}
\begin{figure*}[t]
   \centering
   \includegraphics[width=0.99\linewidth]{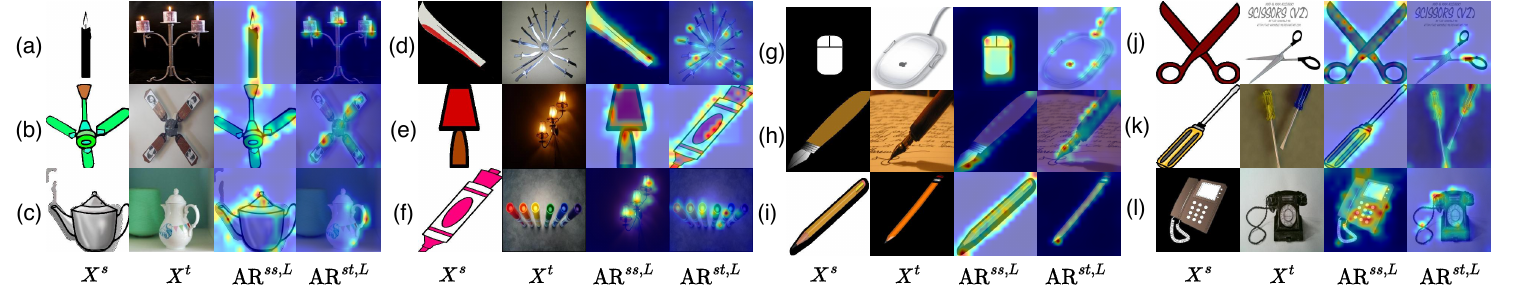}
   \vspace{-4mm}
   \caption{\textbf{Visualization of attention rollout on the Office-Home.} Here, $ {X}^{s}$ and ${X}^{t}$ represent the image data from the source and target domains respectively. $\text{AR}^{ss, L}$ visualizes the feature map of the last layer of self-attention in the source domain, whereas $\text{AR}^{st, L}$ visualizes the feature map of cross-attention in the target domain. PCaM effectively aligns the attention maps of the source and target domains, facilitating feature extraction.}
   \vspace{-3mm}
   \label{tab:fig_feature}
\end{figure*}
In the box interpolation operation, the box identified by the box identification is used to crop the original image. The cropped image features $z^{\text{C}} = z{[a_{<}^{s,l}:a_{>}^{s,l}, a_{\vee}^{s,l}:a_{\wedge}^{s,l}]}$ is then resized to the original size $z^{\text{FR}}$. The box interpolation is defined as,
\begin{equation}
   \small
   z^{\text{FR}}_{n, m}=\sum_{i=0}^1 \sum_{j=0}^1 z^{\text{C}}_{n_0+i, m_0+j} w_{i j}(m,n,m_0,n_0),
   \label{eqn_w}
\end{equation}
where $(m,n)$ are the coordinates of the pixel at the new location, $(m_0,n_0)$ are the coordinates of the closest pixel in the original image, and the interpolation weight $w_{i j}(m,n,m_0,n_0)$ is given by,
\begin{equation}
   \small
   \frac
   {
      \left(\!1\!-\!\left|m-\left(m_0+i\right)\right|\right)\left(1-\left|n-\left(n_0+j\right)\right|\right)
   }
   {
      \!\sum_{k=0}^1 \!\sum_{l=0}^1\!
      \left(1-\left|m\!-\!\left(x_0\!+\!k\right)\right|\right)
      \left(1-\left|n\!-\!\left(n_0\!+\!l\right)\right|\right)
   }.
\end{equation}

\vspace{2mm}
\subsection{Progressively Focused Loss Functions}

Cross-attention maps are often affected by noise, which can reduce the model's ability to focus accurately on critical regions. To address this, we propose a progressively focused loss function, $\mathcal{L}_\text{PF}$, which directs the model to concentrate attention on a continuous, localized area, thereby improving feature extraction from essential regions,
\begin{equation}
   \small
   \begin{aligned}
      \! \mathcal{L}_\text{PF} & \!=\! - \sum_l \sum_{m,n=1}^{N}
      \frac
      {\|\text{AR}_{\mathbf{J}(m,n)}^{st,l}-\text{AR}_{\mathbf{J}(\lfloor m^l_c \rfloor ,\lfloor n^l_c\rfloor)}^{st,l}\|_2^2}
      {\sqrt{(m-m^l_c)^2+ (n - n^l_c)^2}},                                                               \\
      \! m^l_c                 & \!=\! \frac{1}{N^2} \sum_{m,n=1}^N{\!m\text{AR}_{\mathbf{J}(m,n)}^{st,l}},
      n^l_c \!=\! \frac{1}{N^2} \sum_{m,n=1}^N{\!n\text{AR}_{\mathbf{J}(m,n)}^{st,l}},
   \end{aligned}
   \label{eqn:PF}
\end{equation}
where $\mathbf{J}(m,n)$ maps pixel coordinates to patch indices, supporting alignment between pixels and patches. The terms $m^c$ and $n^c$ represent the attention map's center of mass, while $(\lfloor m^c \rfloor, \lfloor n^c \rfloor)$ provides the closest pixel coordinates to this center.

The $\mathcal{L}_\text{PF}$ function progressively guides the model to intensify attention on pixels near the center of mass, reducing sensitivity to noise and ensuring a more focused attention distribution. 
By minimizing the weighted distance between attention values at $(m, n)$ and the center, $\mathcal{L}_\text{PF}$ facilitates a coherent attention pattern that prioritizes key foreground regions, ultimately enhancing feature precision and model robustness. Its differentiability, gradient directionality, and regularization effect are formally proved in \ref{sec:app:progressively_focused_loss}.

\vspace{2mm}
\subsection{The Optimization of PCaM}

The training objective for PCaM is given by,
\begin{equation}
   \mathcal{L} = \mathcal{L}_\text{cls}^{s} + \mathcal{L}_\text{dst} + \mathcal{L}_\text{cls}^{t} + \mathcal{L}_\text{PF},
   \label{eqn:da-all}
\end{equation}
where source domain classification loss $\mathcal{L}_{cls}^{s}$, the distillation loss $\mathcal{L}_{dst}$, and $\mathcal{L}_{cls}^{t}$ target domain classification loss are from CDTrans~\cite{xu2021cdtrans}. The classification loss for the source domain is defined as,
\begin{equation}
   \mathcal{L}_\text{cls}^{s} =
   -\sum_{j = 1}^C \bm{y}^s\log\left(\sigma\left({g\left(\bm{f}^{s}\right)}\right)\right),
   \label{eqn:da-source}
\end{equation}
where $\sigma(\cdot)$ is the softmax active function, $\bm{y} = [y_1, \dots, y_C]$ is a set of label vectors. And, the $\bm{f}^s$ is the feature of data $x^s\in \mathcal{X}^{s}$ in the source domain.

The ViT is encouraged to learn the robust class feature, which is consistent with the content ($\bm{Q}$) and invariant to the style ($\bm{K}/\bm{V}$).
For $\bm{f}^{s\rightarrow t}$ and $\bm{f}^t$, the ground-truth label is not accessible, as they encode the class feature from the target query. Nevertheless, $\bm{f}^{s\rightarrow t}$ can carry some reliable class information because it is translated with the source style (K/V), and $\bm{f}^s$ is trained with the ground-truth label. Therefore, the prediction of $\bm{f}^{s\rightarrow t}$ can be used as the teacher to supervise the prediction of $\bm{f}^t$,
\begin{equation}
   \mathcal{L}_{dst} =
   -\sum_{j = 1}^C \sigma(g(\bm{f}^t) / \tau) \log\left(\sigma\left(g\left(\bm{f}^{s\rightarrow t}\right) / \tau\right) \right),
   \label{eqn:da-target}
\end{equation}
where $\tau > 0$ is the temperature for distillation~\cite{wang2025multi}. Finally, the pseudo labeling strategy, e.g., SHOT~\cite{liang2020we} can be used to learn the target feature $\bm{f}^t$,
\begin{equation}
   \mathcal{L}_{cls}^t =
   -\sum_{j = 1}^C \tilde{\bm{y}}^t\log(\sigma(g(\bm{f}^t))),
   \label{eqn:da-pseudo}
\end{equation}
where $\tilde{\bm{y}}^t$ is the pseudo label for $\bm{x}^t$.
When model inference. Directly feed test data into ViT without additional matching. This is because the PCaM proposed in this paper is only used to bootstrap stable representations that are invariant to the learning domain without affecting the model.

\section{Experiments}
\label{sec:experiments}
\textbf{Datasets.} Our research employs widely-used datasets for domain adaptation, including popular benchmarks such as Office-Home~\cite{venkateswara2017Deep}, VisDA-2017~\cite{peng2017visda}, and DomainNet~\cite{peng2018moment}. The details are shown in the appendix.

\subsection{Implementation Details.}
To ensure comparability with existing SOTA methods, all the backbones are initialized with ImageNet-1K pretrained weights. During training, we use SGD algorithm~\cite{smith2024origin} with a momentum of 0.9 and a weight decay ratio of 1e-4. The details are shown in appendix.

\vspace{2mm}
\subsection{Exploratory Experiment}

To investigate whether the proposed scheme works, we designed the Exploratory Experiment, which consists of two parts to demonstrate the work of Attention Rollout and Feature Refinement.

\textbf{Attention Rollout~(AR).}
The cross-attention rollouts of the Office-Home dataset are shown in Fig.~\ref{tab:fig_feature}. Each set contains four images: $X^s$, $X^t$, $\text{AR}^{ss,L}$, and $\text{AR}^{st,L}$. We observe that the cross-attention is well-focused on the foreground objects, which indicates that our concentration loss can work well to focus attention on a local range. Cropping on such a basis can effectively exclude the noise factor in the background. For example, in Fig.~\ref{tab:fig_feature}j, the presence of advertising text in the scissor image of the target domain may interfere with cross-domain recognition. PCaM can exclude this type of interference by rolling out cross-attention.

\textbf{Feature Refinement~(FR).}
We further examined the progression of feature refinement during training. As shown in Fig.\ref{tab:fig_Exploratory_Experiment}a, we tracked the percentage of cropped pixels on the VisDA-2017 dataset. The results reveal a trend where model accuracy steadily improves as training progresses, coinciding with an increasing focus on the central regions. This pattern suggests that the focus loss is effectively guiding the model's attention toward the foreground, thereby enhancing accuracy. Fig.\ref{tab:fig_Exploratory_Experiment}b provides a specific example: the PCaM's crop boundary (yellow box) gradually centers on the truck in the foreground under the influence of cross-attention, successfully excluding distracting background elements like pedestrians. This refinement underscores the role of PCaM in progressively enhancing recognition accuracy by isolating relevant foreground information.

\begin{table*}[t]
	\footnotesize
	\centering
	\caption{\textbf{Comparison of our approach on Remote Sensing dataset~(AID$\to$NWPU).}
		The highest acc is highlighted in \textbf{bold}. AVE is the average performance. The PCaM method outperforms the baseline by \textbf{4.9\%}.
	}
	\begin{tabular}{l|c c c c c |cc}
		\toprule
		Category                 & UDA  & CMU  & I-UAN         & MA            & CDTrans       & PCaM          \\
		\midrule
		Farmland (Fa.)           & 91.7 & 89.6 & 96.0          & 93.9          & 93.9          & \textbf{96.3} \\
		Forest (Fo.)             & 92.6 & 83.9 & \textbf{96.3} & 87.4          & \textbf{95.4} & \textbf{95.4} \\
		Desert (DR.)             & 78.3 & 80.4 & 78.9          & 64.1          & 90.0          & \textbf{90.9} \\
		River (Ri.)              & 71.3 & 76.4 & 74.9          & 78.6          & 77.0          & \textbf{86.0} \\
		Park (Pa.)               & 81.1 & 84.0 & 86.3          & 89.0          & 94.1          & \textbf{98.4} \\
		Industrial area (In.)    & 77.3 & 71.0 & 75.9          & \textbf{83.4} & 82.3          & 83.0          \\
		Beechwood (Be.)          & 91.6 & 88.6 & 94.9          & 95.4          & 98.7          & \textbf{99.4} \\
		Medium Residential (MR) & 70.0 & 79.3 & 83.0          & \textbf{88.0} & 74.3          & 82.0          \\
		Sparse Residential (SR) & 83.0 & 87.0 & 86.6          & 82.5          & 94.7          & \textbf{95.1} \\
		Airport (Ai.)            & 45.1 & 39.0 & 51.0          & 58.0          & 71.0          & \textbf{82.6} \\
		Bridge (Br.)             & 81.3 & 86.7 & 90.3          & 94.7          & 95.9          & \textbf{97.7} \\
		Baseball field (Ba.)     & 74.6 & 73.3 & 78.7          & 79.0          & 85.6          & \textbf{92.6} \\
		Church (Ch.)             & 73.3 & 74.0 & \textbf{83.4} & \textbf{76.3} & 65.0          & 82.7          \\
		Dense Residential (De.)  & 78.7 & 72.0 & 80.7          & 79.9          & 85.7          & \textbf{86.3} \\
		Meadow (Me.)             & 80.1 & 75.6 & 88.4          & \textbf{91.6} & 45.7          & 56.6          \\
		Mobile Home Park (Mo.)   & 85.4 & 89.7 & 81.7          & 70.7          & \textbf{97.7} & \textbf{97.7} \\
		Runway Strip (RS)        & 82.1 & 87.4 & 84.1          & 84.9          & 84.3          & \textbf{88.9} \\
		Storage Tank (St.)       & 50.0 & 60.4 & 81.0          & 46.0          & 94.3          & \textbf{96.4} \\
		Stadium (ST)             & 88.4 & 94.0 & 95.3          & 92.9          & 96.2          & \textbf{97.6} \\
		Commercial area (Co.)    & 37.4 & 54.1 & 57.0          & 60.1          & 53.0          & \textbf{65.4} \\
		\midrule
		AVE                      & 75.7 & 77.3 & 82.2          & 79.8          & 83.7          & \textbf{88.6} \\
		$\Delta$        &            &         &         &        &                         & \textbf{(+4.9)}      \\
		\bottomrule
	\end{tabular}
	\label{remotesensing}
\end{table*}
\begin{table*}[h]
	\centering
	\caption{\textbf{Ablation studies on VisDA-2017.} The highest acc is highlighted in \textbf{bold}. PCaM outperforms the baseline by \textbf{3.0\%}.}
	\footnotesize
	{
		\begin{tabular}{l|c c c c c c c}
			\toprule
			Category & Baseline      & +Crop  & +Weight & +Det   & +Det          & +Det          & PCaM            \\
			         & (CDTrans)     & Fore.  & Patches &        & +AR           & +AR+FR        & (Full model)    \\
			\midrule
			plane    & 97.7          & 97.9   & 98.1    & 97.9   & 97.9          & \textbf{98.5} & 98.1            \\
			bcycl    & 86.4          & 88.1   & 87.9    & 89.6   & 90.1          & 91.3          & \textbf{93.2}   \\
			bus      & 86.9          & 87.3   & 87.1    & 87.2   & 87.6          & \textbf{90.3} & 90.1            \\
			car      & 83.3          & 84.3   & 84.3    & 83.5   & 83.7          & 84.3          & \textbf{89.4}   \\
			horse    & 97.8          & 97.7   & 98.1    & 97.8   & 97.9          & 98.4          & \textbf{98.8}   \\
			knife    & 97.2          & 97.3   & 97.6    & 97.1   & 97.1          & 96.5          & \textbf{97.3}   \\
			mcycl    & 95.9          & 95.8   & 95.3    & 95.9   & 95.3          & 94.4          & \textbf{96.0}   \\
			person   & 84.1          & 84.2   & 85.1    & 84.3   & 84.1          & 80.1          & \textbf{84.6}   \\
			plant    & \textbf{97.9} & 96.4   & 95.1    & 98.0   & \textbf{98.6} & \textbf{98.0} & \textbf{98.0}   \\
			sktbrd   & 83.5          & 81.9   & 82.1    & 84.5   & 91.2          & \textbf{93.4} & 92.1            \\
			train    & 94.6          & 94.5   & 94.6    & 94.7   & 93.7          & \textbf{95.6} & 95.4            \\
			truck    & 55.3          & 56.2   & 56.9    & 56.3   & 59.1          & 62.2          & \textbf{63.8}   \\ \midrule
			AVE      & 88.4          & 88.5   & 88.5    & 88.9   & 89.7          & 90.2          & \textbf{91.4}   \\
			$\Delta$ & ---           & (+0.1) & (+0.1)  & (+0.5) & (+1.3)        & (+1.8)        & \textbf{(+3.0)} \\
			\bottomrule
		\end{tabular}
	}
	\label{tab_abs}
\end{table*}

\vspace{2mm}
\subsection{Comparision with SOTAs on UDA benchmarks}

We compare our PCaM to various UDA methods, such as CGDM(2021)~\cite{Zhekai2018Zhekai}, DOT(2022)~\cite{DOT_2022_mm}, CDTrans(2023)~\cite{xu2021cdtrans}, DsiT(2023)~\cite{CDSiT_2023_ICCV}, DPP\&BST~\cite{luo2025explicitly}, and C-SFTrans(2024)~\cite{C_SFTrans_2023}  with the the backbones of DeiT-base. We used ImageNet-1k as our pre-training dataset for each model. Some methods that utilize stronger pre-trained models are not included in the comparison~\cite{hoyer2023mic,ELS_2023_ICLR}.

\textbf{Comparision on Classic UDA Datasets (Office-Home).} PCaM trains a DeiT base directly on the source domain and pseudo-label. We conducted comparative experiments to verify the effectiveness of PCaM on Office-Home, as shown in Table~\ref{tab:officehome}, where we achieve comparable and even better performance compared to SOTA UDA methods. In addition, PCaM also outperformes CDTrans (+0.8\%) due to the contribution of the proposed PCaM method. The average accuracy of our method is 81.3\%, which achieves the best performance on six subsets when applied to CDTrans in Office-Home.

\begin{figure}
   \centering
   \includegraphics[width=0.70\linewidth]{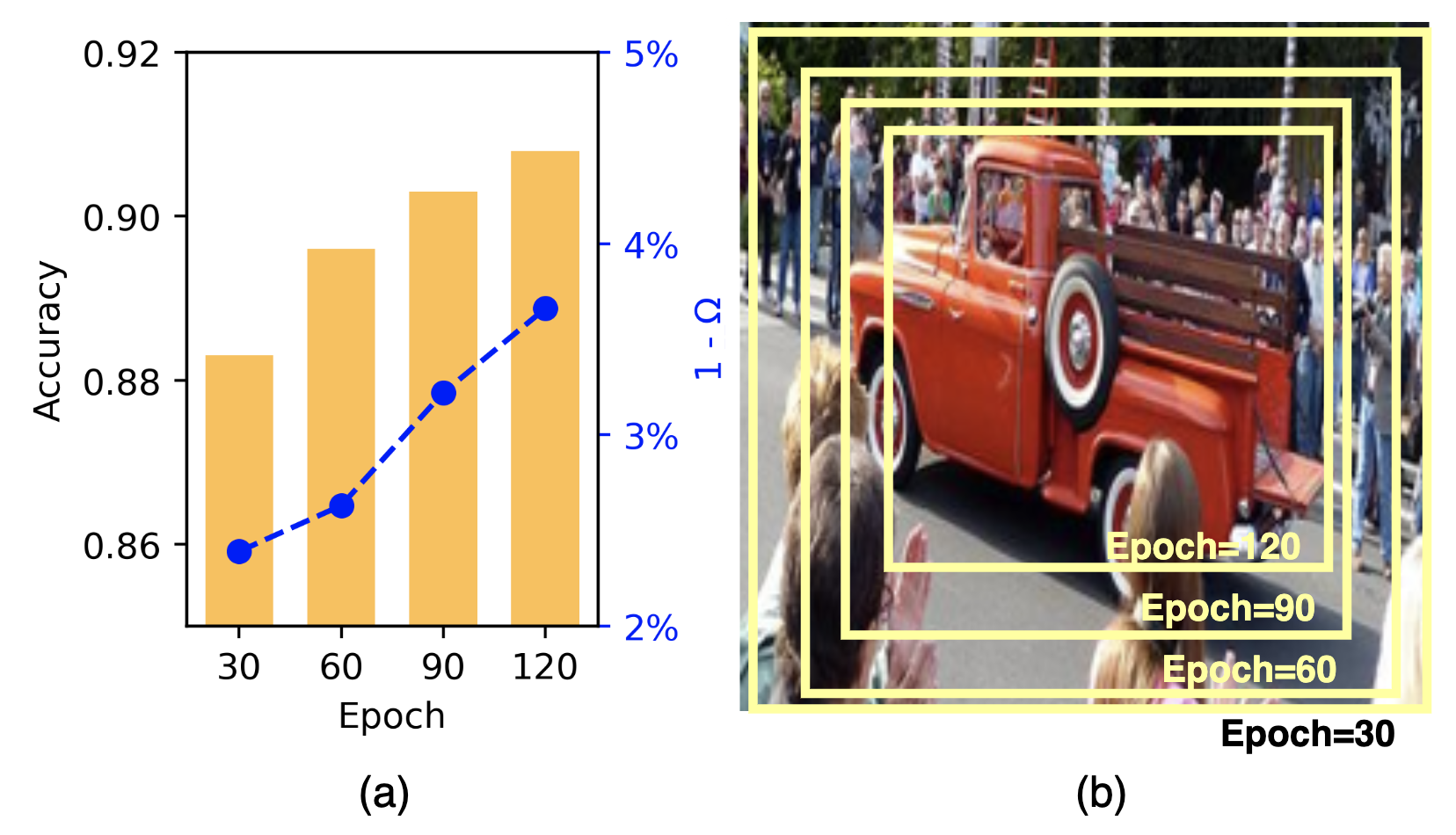}
   \caption{\textbf{Exploratory experiment of PCaM.} (a) Barplot and line plot illustrating performance as the number of training epochs increases. (b) The cross-attention mechanism of PCaM highlights regions corresponding to different epoch frames.}
   \label{tab:fig_Exploratory_Experiment}
\end{figure}

\textbf{Comparison on UDA Datasets with Large Domain Differences (VisDA-2017).}
As shown in Table~\ref{tab:visda}, our proposed PCaM achieves the highest average accuracy of {91.4\%} on the VisDA-2017 dataset, outperforming the previous best method C-SFTrans (2024) by {1.1\%} and the widely adopted CDTrans (2023) by a significant margin of {3.0\%}. This result highlights the effectiveness of our method in addressing challenging domain adaptation settings with substantial domain shifts.
Further examination of per-class performance reveals notable gains in categories that are severely affected by Foreground Object Mismatch (FOM), such as “\texttt{sktbrd}” and “\texttt{bcycl}.” Specifically, PCaM achieves an accuracy of {92.1\%} on “\texttt{sktbrd},” representing an absolute improvement of {8.6\%} over CDTrans (83.5\%), and {93.2\%} on “\texttt{bcycl},” a {6.8\%} improvement over CDTrans (86.4\%). These categories typically suffer from small foreground objects in cluttered scenes, leading to background-dominated attention in prior methods.
By progressively refining attention to foreground regions and introducing a concentration-aware loss, PCaM mitigates the influence of background noise and improves feature alignment across domains. Overall, these results confirm that PCaM not only outperforms ViT-based UDA baselines but also provides stronger generalization on categories sensitive to spatial mismatch.

\begin{figure}
   \centering
   \includegraphics[width=0.60\linewidth]{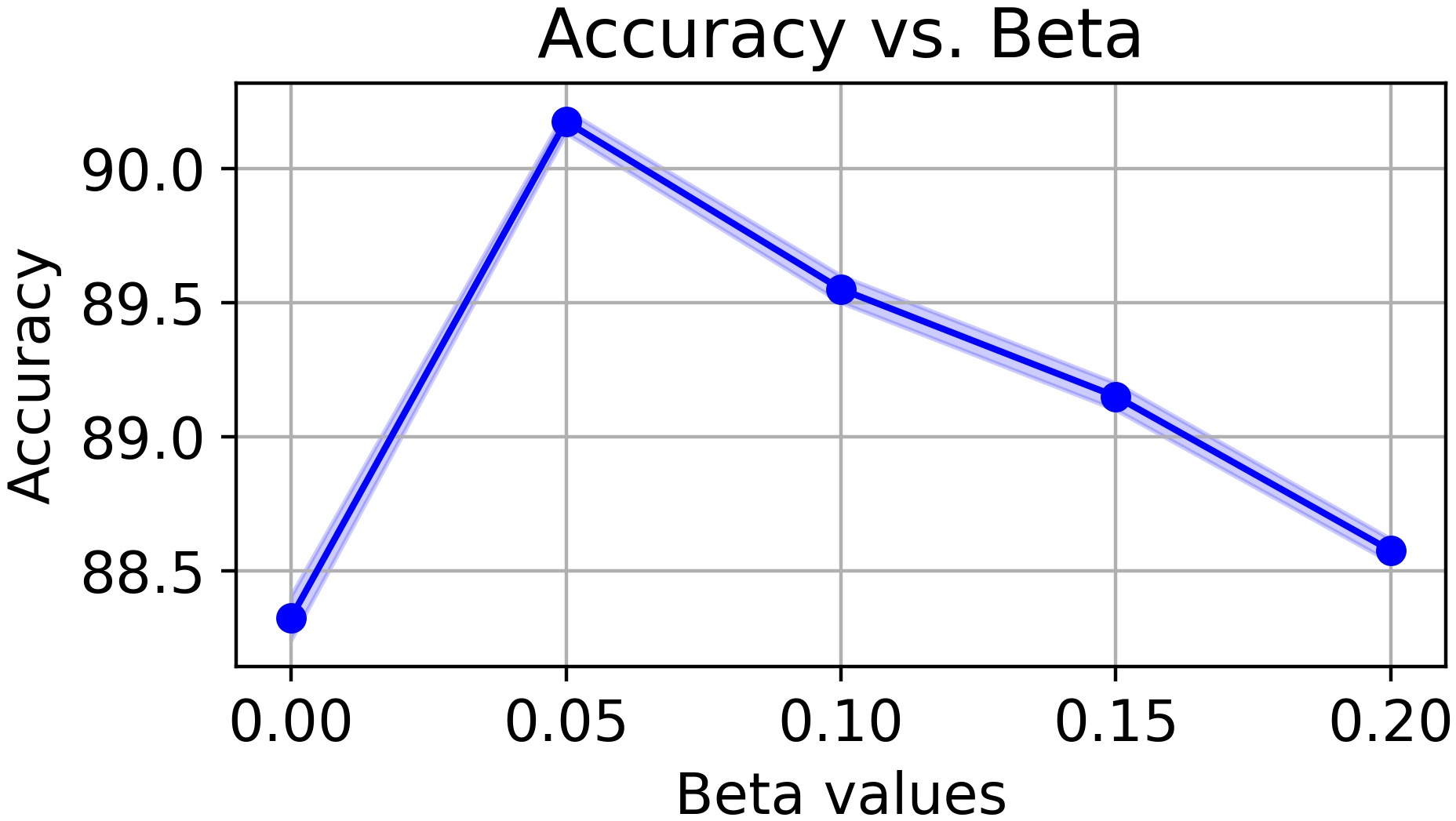}
   \caption{{\textbf{Parameter analysis.} The parameter analysis of parameter $\beta$~(Beta) in Eq.~(\ref{eqn_attn}) on VisDA-2017.} The performance of PCaM is stable when $\beta$ is around 0.05.}
   \label{tab:sensitive}
\end{figure}

\textbf{Greater Advantage on More Complex UDA Datasets~(DomainNet).} On large datasets, the average performance of our methods on DomainNet is still significantly better than other approaches, as shown in Fig.~\ref{fig:domainet}. Our method has an average accuracy of 47.2\%, which exceeds the SOTA results and outperforms CDTrans (+2.0\%); meanwhile, we get about 3.8\% better accuracy compared to DOT-B. From this result, our model can still adapt more accurately in the target domain when faced with such a large dataset. We confirm that our proposed method works successfully for UDA tasks. This shows that PCaM is generally applicable to different domain adaptation tasks.

\textbf{Comparision on Remote Sensing Dataset.} To verify the performance of PCaM on remote sensing datasets, we performed additional experiments using the AID~\cite{xia2017aid}$\to$NWPU~\cite{wang2020nwpu} dataset. The baseline methods include UDA(2019)~\cite{UDA}, CMU(2020)~\cite{CMU}, I-UAN(2021)~\cite{I-UAN}, CDTrans(2023)~\cite{xu2021cdtrans}, and MA(2023)~\citet{Xu_2023}, The results are shown in Table.~\ref{remotesensing}, where we can see that PCaM largely outperforms both the SOTA method~(MA~\cite{Xu_2023}) and our baseline method. This suggests that PCaM is able to show advantages in the more severe problems of the FOM problem.

\textbf{Minor Increase in Computational Cost Achieves Significant Performance Gains} The computational cost results, illustrated in Fig.~\ref{tab_cost}, show that PCaM incurs a slight increase in computation time compared to baseline methods. However, this modest trade-off yields substantial performance improvements, making it advantageous for applications that prioritize accuracy.
\begin{table}[t]
    \centering
    \caption{\textbf{The computational cost of PCaM on VisDA \& DomainNet.} PCaM does have a slightly higher computational cost than baseline methods. However, the notable improvement in performance justifies this minor increase. (m: minutes; s: seconds)}
    \begin{tabular}{@{}l|cc|cc@{}}
        \toprule
                & \begin{tabular}[c]{@{}c@{}}Train\\(10 epoch)\end{tabular} & \begin{tabular}[c]{@{}c@{}}Test\\(14k sample)\end{tabular} & \begin{tabular}[c]{@{}c@{}}VisDA\\ (Acc)\end{tabular} & \begin{tabular}[c]{@{}c@{}}DomainNet\\ (Acc)\end{tabular} \\ \midrule
        CDtrans & \textbf{9m45s}                                          & \ \ \ \ \ \ \ \textbf{21s}                                 & 88.4                                                  & 45.2                                                      \\
        PCaM    & 10m06s(+21s)                                            & \ \ \ \ \ \ \ \textbf{21s}                                 & \textbf{91.4(+3.0)}                                   & \textbf{47.2(+2.0)}                                       \\
        \bottomrule
    \end{tabular}
    \label{tab_cost}
    \vspace{-3mm}
\end{table}

\vspace{2mm}
\subsection{Ablation Study \& Discustions}

\textbf{Overall Ablation Study.} We trained a DeiT-base model on both the source and target domains using the pseudo-labeling method~\cite{liang2020we,xu2021cdtrans} and conducted comparative experiments on the VisDA-2017 dataset (see Table \ref{tab_abs}) to evaluate the performance of the proposed PCaM framework. The results show that PCaM achieved a significant improvement over CDTrans, with a performance gain of +3.0\%, primarily attributed to the contributions of the PCaM modules. We also conducted ablation studies based on the PCaM framework to evaluate the contributions of individual components such as Detection Enhancement (Det), Attention Refinement (AR), and Feature Refinement (FR). The results demonstrate incremental improvements from each component, confirming the robustness and effectiveness of the framework. These findings validate that the PCaM framework achieves SOTA performance on the VisDA-2017 dataset.

\textbf{Existence  of Foreground Mismatch.} On the `bcycl', `bus', `sktbrd', and `truck' classes of VisDA-2017, there is more background information in the target dataset, which causes a mismatch between the source domain and target domain. CDTrans is less effective on these four datasets, and our method enhances feature learning by gradually focusing on the same class of target objects. The alignment of features is enhanced by aligning the truck of the source domain with the truck of the target region, removing interfering information in the target domain. There are many foreground mismatch cases in the VisDA-2017 dataset, and our PCaM significantly improves these datasets. For example, as shown in Table~\ref{tab_abs_visda}, our method improves by 5.7\% on the `bcycl' dataset, 3.2\% on the `bus' dataset, 8.5\% on the `sktbrd' dataset, and by 7.5\% on the `truck' dataset. The experiments further demonstrate that our PCaM aligns features of the same class of targets, improving the problem of foreground mismatch.

\textbf{Stability \& Parameter Analysis.} In our method, the progressively focused cross-attention mechanisms are used to guide the generation of feature refinement. The threshold $\beta$ determines the size of the localization box, with a larger $\beta$ resulting in a smaller box. However, performance degrades rapidly as $\beta$ increases above 0.2. We suspect that smaller bounding boxes dramatically reduce the variance of the views, making it trivial to learn discriminative features. We tried different thresholds; the results are shown in Fig.~\ref{tab:sensitive}.

\textbf{PCaM Beyond directly Cropping Foreground \& Weight Patches.}
As shown in Table.~\ref{tab_abs}, PCaM mitigates the FOM problem by adapting to the foreground by gradually focusing during training, in order to demonstrate the novelty of this scheme. We tested two more direct schemes, CDTrans+Crop Foreground directly using the foreground recognition method \cite{wang2025multidimensional}, and CDTrans +Weight Patches directly weighting Patches according to the absolute value of AR. The results show that both methods are able to deliver a boost, which is a side effect of the fact that the FOM problem does exist. However, the boost brought by both methods is not obvious, which indicates that PCaM can solve this problem better.

\begin{figure}[t]
   \centering
   \includegraphics[width=0.69\linewidth]{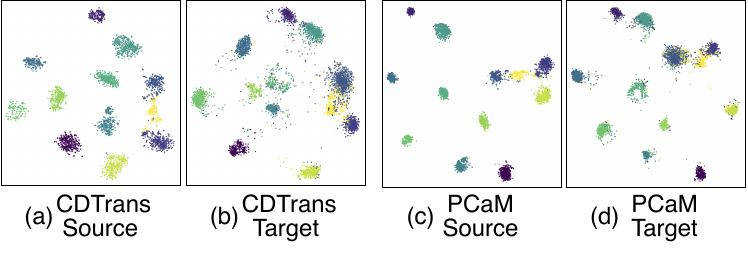}
   \caption{\textbf{Scatter visualization of CDTrans and PCaM\@}: t-SNE visualization of domain alignment effects. PCaM shows better domain alignment than CDTrans.}
   \label{tab:fig_Experiment}
\end{figure}

\textbf{Domain Alignment \& t-SNE visualization.}
In Fig.~\ref{tab:fig_Experiment}, we illustrate the feature distributions learned by CDTrans and PCaM using t-SNE embeddings on the VisDa-2017 dataset, based on the activations from the final fully connected layer. Compared to CDTrans, PCaM achieves clearer separation of target classes, with larger inter-class margins and tighter intra-class clusters. This underscores PCaM's effectiveness as a robust solution for unsupervised domain adaptation.

\subsection{Pseudo-Label Robustness Ablation}
\label{sec:4exper}
To assess the robustness of our method under noisy pseudo labels, we design a pseudo-label perturbation experiment. Specifically, we simulate various degrees of label noise by randomly altering a proportion of the pseudo labels used for pairing in the target domain. The corruption ratio $\gamma$ denotes the pseudo-label accuracy retained after noise injection.

We compare PCaM with CDTrans under pseudo-label accuracies of $\gamma \in {100\%, 90\%, 80\%, 70\%, 60\%, 50\%}$. As shown in Fig.X and TableX, PCaM consistently outperforms CDTrans under all corruption levels. Even when only 50\% of the pseudo labels are correct, PCaM maintains a relatively high accuracy, demonstrating its robustness against label noise.

This robustness comes from the dual-filtering strategy in PCaM, which combines both pseudo-label agreement and feature similarity thresholding to mitigate the impact of erroneous sample pairs.

\begin{table}[t]
   \centering
   \caption{\textbf{Robustness under pseudo-label noise.} 
   Accuracy (\%) on VisDA-2017 with varying pseudo-label accuracy $\gamma$. 
   PCaM exhibits better robustness than CDTrans under severe pseudo-label noise.}
   \label{tab:robustness}
   \vspace{1mm}
   \begin{tabular}{l|cccccc}
       \toprule
       \textbf{Method} & \textbf{100\%} & \textbf{90\%} & \textbf{80\%} & \textbf{70\%} & \textbf{60\%} & \textbf{50\%} \\
       \midrule
       CDTrans         & 88.4 & 86.7 & 84.1 & 81.0 & 76.4 & 70.2 \\
       \textbf{PCaM}   & \textbf{91.4} & \textbf{90.6} & \textbf{89.1} & \textbf{86.5} & \textbf{84.0} & \textbf{79.6} \\
       \bottomrule
   \end{tabular}
\end{table}

\vspace{-2mm}
\section{Conclusion}
\label{sec:conclusion}
In conclusion, this paper presents a novel method, Progressively Focused Cross-Attention Mechanisms (PCaM), for addressing the problem of noisy pseudo-labels in category-level-based alignment for Unsupervised Domain Adaptation (UDA). While PCaM may slightly increase the computational time, it extracts foreground information and removes irrelevant background information, effectively aligning the source and target domains' features and improving the scale mismatch problem. The experimental results show that our method outperforms several SOTA UDA methods, including VisDA-2017 and DomainNet benchmarks. 

{
    \small
    \bibliographystyle{plain}
    \bibliography{egbib}
}

\appendix

\tableofcontents

\section{Related Works}

\textbf{Unsupervised Domain Adaptation.}
Unsupervised Domain Adaptation (UDA) tackles the challenge of transferring knowledge to a target domain lacking labeled data by leveraging innovative methods to bridge the domain gap. FixBi~\cite{Jaemin2021FixBi} enhances cross-domain alignment by introducing multiple intermediate domains, while CDTrans~\cite{xu2021cdtrans} demonstrates robustness against noisy labels and excels in feature alignment. DOT~\cite{DOT_2022_mm} employs a transformer-based approach, using categorical tokens from both source and target domains to effectively capture domain-invariant and domain-specific features. PMTrans~\cite{PTrans2023} addresses the domain gap through game-theoretic strategies, connecting source and target domains via intermediary domains. Similarly, SAN~\cite{zang2023boosting} leverages unsupervised contrastive learning to handle open-set and partial-set domain adaptation challenges. Despite these advancements, existing UDA methods often falter in scenarios involving foreground object mismatches, which can lead to unreliable and error-prone outcomes.

\textbf{Cross-Attention Mechanism.}
The attention mechanism plays a critical role in sequence modeling and transduction tasks, demonstrating remarkable success across diverse applications such as abstractive summarization, reading comprehension, textual entailment, and learning task-independent sentence representations~\cite{chen2025self}. Building on this foundation, cross-attention modules extend the capability of attention by processing inputs from two distinct modalities, effectively aligning and aggregating information between them. This approach has proven invaluable in tasks like vision-to-text~\cite{Yao_2019_ACL}, and vision-to-vision~\cite{Li_2021_CVPR}, enabling more effective integration and understanding of multimodal data.

\section{Theoretical Analysis of Attention Rollout} \label{sec:app:attention_rollout}

In this section, we theoretically analyze the attention rollout (AR) mechanism used in PCaM and prove several important properties: (1) boundedness and stability across layers; (2) convergence of the cumulative attention; and (3) the tendency to accumulate attention on semantically consistent foreground patches.

We consider the cross-attention rollout at layer $l$ between source patch $i$ and target patch $j$, defined recursively as:
\begin{equation}
\label{eq:ar_def}
\bar{\text{AR}}_{i,j}^{l} =
\begin{cases}
\sigma\left(\frac{\bm{z}_{i}^{s,1} \cdot \bm{z}_{j}^{t,1}}{\sqrt{d}}\right), & l = 1, \\
\bar{\text{AR}}_{i,j}^{l-1} + \sigma\left(\frac{\bm{z}_{i}^{s,l} \cdot \bm{z}_{j}^{t,l}}{\sqrt{d}}\right), & l > 1,
\end{cases}
\end{equation}
where $\sigma(\cdot)$ denotes the softmax activation over all $j \in \{1, \ldots, N\}$ for fixed $i$.

We now present three core results.

\begin{lemma}[Boundedness and Non-negativity]
\label{lemma:boundedness}
For all $l \geq 1$ and all $i, j \in \{1, \ldots, N\}$, the attention rollout satisfies:
\[
0 \leq \bar{\text{AR}}_{i,j}^{l} \leq l.
\]
\end{lemma}

\begin{proof}
We prove this by induction on $l$.

\textbf{Base case:} For $l=1$,
\[
\bar{\text{AR}}_{i,j}^{1} = \sigma\left(\frac{\bm{z}_i^{s,1} \cdot \bm{z}_j^{t,1}}{\sqrt{d}}\right) \in [0,1],
\]
since softmax outputs lie in $[0,1]$ and sum to $1$ over $j$.

\textbf{Inductive step:} Suppose $\bar{\text{AR}}_{i,j}^{l-1} \in [0, l-1]$. Then:
\[
\bar{\text{AR}}_{i,j}^{l} = \bar{\text{AR}}_{i,j}^{l-1} + \sigma\left(\frac{\bm{z}_i^{s,l} \cdot \bm{z}_j^{t,l}}{\sqrt{d}}\right) \leq (l - 1) + 1 = l,
\]
and clearly remains non-negative. Hence, by induction, the lemma holds for all $l$.
\end{proof}

\begin{lemma}[Normalized Rollout Convergence]
\label{lemma:normalization}
Define the normalized rollout map:
\[
\hat{\text{AR}}_{i,j}^{l} = \frac{1}{l} \bar{\text{AR}}_{i,j}^{l}.
\]
Then $\hat{\text{AR}}_{i,j}^{l} \in [0,1]$, and as $l \to \infty$, the normalized rollout converges to the mean attention:
\[
\lim_{l \to \infty} \hat{\text{AR}}_{i,j}^{l} = \lim_{l \to \infty} \frac{1}{l} \sum_{k=1}^{l} \sigma\left(\frac{\bm{z}_{i}^{s,k} \cdot \bm{z}_{j}^{t,k}}{\sqrt{d}}\right).
\]
\end{lemma}

\begin{proof}
Each term in the sum is bounded between $[0,1]$, and the total sum is divided by $l$, so:
\[
0 \leq \hat{\text{AR}}_{i,j}^{l} \leq 1.
\]
The right-hand side is a Cesàro average of bounded terms, and hence converges by the Cesàro mean convergence theorem (since the sequence is bounded). Thus, $\hat{\text{AR}}_{i,j}^{l}$ converges.
\end{proof}

\begin{proposition}[Foreground Aggregation Trend]
\label{prop:foreground_focus}
If the source query patch $i$ corresponds to the foreground region, and the source-to-target attention maps $\sigma(\bm{z}_i^{s,l} \cdot \bm{z}_j^{t,l})$ concentrate on foreground patches $j$ in target, then $\hat{\text{AR}}_{i,j}^{l}$ tends to assign higher values to foreground $j$ than background $j$.
\end{proposition}

\begin{proof}[Sketch]
Given that ViT attention is query-key dot product based, and assuming $\bm{z}_i^{s,l}$ (foreground query) aligns semantically with target foreground patches $\bm{z}_j^{t,l}$, their dot product will be larger than with background patches. After softmax, this yields higher attention weights to foreground patches $j$.

Since AR aggregates over layers, this foreground emphasis is reinforced, i.e.,
\[
\forall j \in \text{foreground},\quad \hat{\text{AR}}_{i,j}^{l} \gg \hat{\text{AR}}_{i,j'}^{l} \quad \text{for } j' \in \text{background}.
\]
\end{proof}

\vspace{1mm}
\noindent\textbf{Conclusion.} These results show that AR is a stable and bounded mechanism for accumulating cross-domain attention across transformer layers. When initialized from meaningful source queries (e.g., foreground tokens), the rollout map reliably emphasizes semantically consistent foreground regions in the target domain, supporting our use of AR for foreground-focused alignment in domain adaptation.

\section{Theoretical Analysis of the Progressively Focused Loss}\label{sec:app:progressively_focused_loss}

In this section, we provide a theoretical analysis of the progressively focused loss $\mathcal{L}_\text{PF}$ defined in Eq.~\eqref{eqn:PF}. We show that this loss (1) is differentiable with respect to the attention rollout map, (2) promotes spatial smoothness and central aggregation, and (3) encourages foreground-attention concentration through distance-based weighting.

Recall the loss function:
\begin{equation}
\label{eq:pf_repeat}
\mathcal{L}_\text{PF} = - \sum_{m,n=1}^{N}
\frac
{\left\|\text{AR}_{\mathbf{J}(m,n)} - \text{AR}_{\mathbf{J}(m^c,n^c)}\right\|_2^2}
{\sqrt{(m - m^c)^2 + (n - n^c)^2}},
\end{equation}
where $(m^c, n^c)$ is the center of mass of the attention rollout map $\text{AR}$:
\begin{align}
m^c &= \frac{1}{N^2} \sum_{m,n=1}^{N} m \cdot \text{AR}_{\mathbf{J}(m,n)}, \\
n^c &= \frac{1}{N^2} \sum_{m,n=1}^{N} n \cdot \text{AR}_{\mathbf{J}(m,n)}.
\end{align}
$\mathbf{J}(m,n)$ maps pixel coordinates to patch index.

We now present the main results.

\begin{lemma}[Differentiability]
\label{lemma:diff}
The progressively focused loss $\mathcal{L}_\text{PF}$ is continuously differentiable with respect to the attention rollout values $\text{AR}_{\mathbf{J}(m,n)}$, assuming the attention map is non-degenerate (i.e., no zero denominator).
\end{lemma}

\begin{proof}
The loss $\mathcal{L}_\text{PF}$ is composed of element-wise differentiable operations: subtraction, $L_2$ norm, and inverse Euclidean distance (excluding singularity at $(m, n) = (m^c, n^c)$). The denominator 
\[
d_{m,n} := \sqrt{(m - m^c)^2 + (n - n^c)^2}
\]
is differentiable almost everywhere and strictly positive for $(m,n) \neq (m^c, n^c)$. In practice, the center of mass is a real-valued point, while $(m,n)$ takes integer values, so equality holds with probability zero.

Moreover, $(m^c, n^c)$ is a weighted average of $\text{AR}$ values and hence differentiable with respect to them. Thus, $\mathcal{L}_\text{PF}$ is a composition of differentiable functions and remains differentiable almost everywhere.
\end{proof}

\begin{proposition}[Gradient Direction Promotes Central Aggregation]
\label{prop:center_pull}
The gradient of $\mathcal{L}_\text{PF}$ with respect to attention values $\text{AR}_{\mathbf{J}(m,n)}$ encourages the values to approach that of the center $(m^c,n^c)$, especially for pixels closer to the center.
\end{proposition}

\begin{proof}[Sketch]
Denote $A_{m,n} := \text{AR}_{\mathbf{J}(m,n)}$ and $A_{c} := \text{AR}_{\mathbf{J}(m^c,n^c)}$. Then:
\[
\frac{\partial \mathcal{L}_\text{PF}}{\partial A_{m,n}} 
= -2 \cdot \frac{(A_{m,n} - A_c)}{d_{m,n}} 
+ \text{additional terms involving } \frac{\partial A_c}{\partial A_{m,n}}, \frac{\partial m^c}{\partial A_{m,n}}.
\]

The first term dominates and exhibits a "pulling" effect: when $A_{m,n} > A_c$, the gradient is negative, reducing $A_{m,n}$; when $A_{m,n} < A_c$, the gradient is positive, increasing $A_{m,n}$. Therefore, attention values are guided toward the central value.

Moreover, this effect is stronger when $d_{m,n}$ is small (i.e., pixels close to the center), since the denominator makes the gradient magnitude larger.
\end{proof}

\begin{corollary}[Local Smoothness Enforcement]
\label{cor:smooth}
The loss $\mathcal{L}_\text{PF}$ penalizes sharp deviations in attention values between nearby pixels and the center, promoting smooth and coherent attention maps.
\end{corollary}

\begin{proof}
The squared difference term $(A_{m,n} - A_c)^2$ is minimized when $A_{m,n}$ is close to $A_c$. Due to the distance-weighted denominator, this penalty is stronger for nearby pixels. Therefore, nearby attention values will converge to similar levels, creating locally smooth attention.
\end{proof}

\vspace{1mm}
\noindent\textbf{Conclusion.} The progressively focused loss is differentiable and stabilizes attention training by encouraging spatially smooth, center-concentrated focus. It acts as a soft regularizer that prevents attention maps from spreading arbitrarily or attending to scattered, noisy regions, aligning with the design goal of foreground-guided alignment.

\section{Stability Theoretical Analysis of the Box Identification Procedure} \label{sec:app:box_identification}

In this section, we provide a theoretical analysis of the box identification (BI) function defined in Eq.~\eqref{eqn_attn_2}, which computes a bounding box from the attention rollout map $\text{AR}_j$ via thresholding with a parameter $\beta \in [0,1]$. We aim to show that the procedure is:

(1) piecewise continuous with respect to $\beta$,

(2) robust to small perturbations in $\text{AR}_j$,

(3) guarantees a lower-bounded foreground coverage under mild assumptions.

Let $\text{AR}^{st,l} \in \mathbb{R}^{N}$ be the flattened patch-wise attention rollout map at layer $l$, and let $I_R(j), I_C(j)$ map patch index $j$ to the row and column index of the original image grid.

Given a fixed threshold $\beta$, the box identification algorithm selects a subset of patch indices:
\[
\mathcal{J}_\beta := \{ j \in \{1,\ldots,N\} \mid \text{AR}_j > \beta \}.
\]
The bounding box coordinates are:
\begin{equation}
\label{eq:bi_bounds}
\begin{aligned}
a_< &= \min_{j \in \mathcal{J}_\beta} I_R(j), \quad 
a_> = \max_{j \in \mathcal{J}_\beta} I_R(j), \\
a_\vee &= \min_{j \in \mathcal{J}_\beta} I_C(j), \quad 
a_\wedge = \max_{j \in \mathcal{J}_\beta} I_C(j).
\end{aligned}
\end{equation}

\begin{lemma}[Piecewise Continuity w.r.t. $\beta$]
\label{lemma:piecewise}
The box coordinates $\{a_<, a_>, a_\vee, a_\wedge\}$ are piecewise constant and right-continuous functions of the threshold $\beta \in [0,1]$.
\end{lemma}

\begin{proof}
$\mathcal{J}_\beta$ is a discrete subset that changes only when $\beta$ crosses an attention value in $\text{AR}_j$. Since there are at most $N$ unique attention values, there are at most $N$ breakpoints in the interval $[0,1]$ where the set $\mathcal{J}_\beta$ changes.

Between breakpoints, $\mathcal{J}_\beta$ remains constant, so the min/max values over $I_R(j)$ and $I_C(j)$ also remain constant. At each breakpoint, as $\beta$ increases, one element $j^*$ is excluded, which may affect the box bounds.

Right-continuity follows because:
\[
\lim_{\epsilon \to 0^+} \mathcal{J}_{\beta + \epsilon} \subseteq \mathcal{J}_\beta.
\]
Thus, the bounds $\{a_<, a_>, a_\vee, a_\wedge\}$ jump downward or stay unchanged as $\beta$ increases.
\end{proof}

\begin{lemma}[Perturbation Robustness]
\label{lemma:noise}
Suppose the attention map $\text{AR}_j$ is perturbed by $\delta_j$ with $|\delta_j| < \epsilon$, and attention values are $(\epsilon, \Delta)$-separated, i.e., the minimum gap between sorted $\text{AR}_j$ values is at least $\Delta > \epsilon$. Then the bounding box is invariant to the perturbation.
\end{lemma}

\begin{proof}
Let $\text{AR}_j'$ be the perturbed map, $\text{AR}_j' = \text{AR}_j + \delta_j$. Since $|\delta_j| < \epsilon < \Delta$, the order of attention values remains unchanged, i.e.,
\[
\text{rank}(\text{AR}_j') = \text{rank}(\text{AR}_j).
\]
Therefore, the set $\mathcal{J}_\beta$ selected via thresholding is unchanged, which implies the box coordinates remain unchanged.
\end{proof}

\begin{proposition}[Minimum Box Area Bound]
\label{prop:min_area}
If the total attention mass is normalized to 1 (i.e., $\sum_j \text{AR}_j = 1$) and the mass within a contiguous region $\mathcal{J}^*$ exceeds $\gamma > \beta$, then the box computed by threshold $\beta$ will include $\mathcal{J}^*$ and hence have area at least $|\mathcal{J}^*|$.
\end{proposition}

\begin{proof}
By definition, all $j \in \mathcal{J}^*$ have $\text{AR}_j \geq \gamma > \beta$, hence they are retained in $\mathcal{J}_\beta$. The bounding box will cover this region since it spans from the min to max row/column index over $\mathcal{J}_\beta$, which includes $\mathcal{J}^*$. Therefore, the area is bounded below by $|\mathcal{J}^*|$.
\end{proof}

\vspace{1mm}
\noindent\textbf{Conclusion.} The attention-based box identification process is stable with respect to small perturbations and threshold tuning. It behaves as a piecewise constant function over $\beta$, and retains the high-attention regions even under noise. This ensures reliable extraction of foreground regions in training and mitigates the impact of spurious or scattered attention.

\section{Baseline Methods}
To ensure a fair comparison, we evaluated several ViT-based UDA methods, including CDTrans and DOT-B, using the results reported in their respective papers. Additionally, we referred to notable convolutional network-based methods, such as FixBi and CDAN, as summarized in Table 1. Notably, even when compared purely in terms of accuracy, our method outperforms these convolutional network-based approaches on the VisDA-2017 benchmark.

\begin{table}[h]
  \label{tab:app:commands}
  \centering
  \caption{\textbf{Comparison with SOTA methods on VisDA-2017~\cite{peng2017visda}.} The best performance is marked as bold.  The avg column contains mean across all classes.}
  \begin{tabular}{c|ccccccccccccl}
    \toprule
    Method               & backbone   & avg           \\
    \midrule
    CAN                  & ResNet-101 & 87.2          \\
    FixBi                & ResNet-101 & 87.2          \\
    CGDM-B               & ViT-B      & 82.3          \\
    SHOT-B               & ViT-B      & 85.9          \\
    CDtrans              & ViT-B      & 88.4          \\
    SDAT                 & ViT-B      & 89.8          \\
    DOT-B                & ViT-B      & 90.3          \\
    PCaM~(\textbf{ours}) & ViT-B      & \textbf{90.8} \\
    \bottomrule
  \end{tabular}
\end{table}

\section{Datasets}
\textbf{Office-Home} is a small yet challenging dataset for domain adaptation, featuring office and home scenes. It includes 15,500 images across 65 categories, organized into four domains: artistic images (Ar), clip art (Cl), product images (Pr), and real-world images (Rw). This dataset captures a variety of everyday objects, making it an essential benchmark for domain adaptation. More details can be found at \href{https://paperswithcode.com/dataset/office-home}{Office-Home Dataset}.

\textbf{VisDA-2017} is a large-scale simulation-to-real dataset characterized by substantial domain differences. The source domain comprises 152,397 synthetic images generated with 3D models under varying angles and lighting conditions, while the target domain consists of 55,388 real-world images. It is widely used to evaluate methods that address significant domain gaps. More information is available at \href{https://paperswithcode.com/dataset/visda-2017}{VisDA-2017 Dataset}.

\textbf{DomainNet} is the largest and most challenging dataset for domain adaptation, containing over 600,000 images spanning 345 categories and six domains: Clipart (clp), Infograph (inf), Painting (pnt), Quickdraw (qdr), Real (rel), and Sketch (skt). By permuting these domains, 30 adaptation tasks can be created (e.g., clp→inf, ..., skt→rel), making it a comprehensive benchmark for testing the robustness of domain adaptation approaches. Additional details can be found at \href{https://paperswithcode.com/dataset/domainnet}{DomainNet Dataset}.

\textbf{Remote Sensing Datasets.}
Remote sensing datasets are essential for evaluating methods in tasks such as image classification and domain adaptation, offering diverse and complex aerial imagery. The AID dataset~\cite{xia2017aid} contains 10,000 high-resolution aerial images spanning 30 categories, including urban, agricultural, and industrial land uses, with image sizes of 600$\times$600 pixels. It is known for its diversity in scale, perspective, and environmental conditions, making it a challenging benchmark. The NWPU-RESISC45 dataset~\cite{wang2020nwpu}, on the other hand, is larger, comprising 31,500 images across 45 scene classes. Each image measures 256$\times$256 pixels, covering a wide range of land-use patterns such as forests, sports fields, and residential areas. These datasets offer significant domain diversity, providing a reliable benchmark for evaluating domain adaptation methods under realistic conditions.

\begin{table*}[h]
  \centering
  \caption{Information of UDA datasets.}
  \label{tab:app:commands}
  \begin{tabular}{lcc}
    \toprule
    Dataset        & Number of Images & Number of Categories \\
    \midrule
    Office-Home    & 15,500           & 65                   \\
    VisDA-2017     & 207,785          & 12                   \\
    DomainNet      & 586,575          & 345                  \\
    Remote Sensing & 10,000           & 30                   \\
    \bottomrule
  \end{tabular}
\end{table*}

\section{Hyperparameters}

To evaluate the effectiveness of our proposed method, we conducted experiments on four popular UDA benchmarks: VisDA-2017, Office-Home, Office-31, and DomainNet. For fair comparison, we adopted the DeiT-base architecture~\cite{touvron2021training} as the backbone. The input image size was fixed at 224$\times$224, differing from ViT-base architectures commonly used in other methods. All backbones were initialized with ImageNet~1K pretrained weights.

We utilized the Stochastic Gradient Descent (SGD) algorithm~\cite{smith2024origin} with a momentum of 0.9 and a weight decay ratio of 1e-4 to optimize the training process. The learning rate was set to 3e-3 for Office-Home, Office-31, and DomainNet, while a higher learning rate of 8e-3 was used for VisDA-2017~\cite{peng2017visda}, which converges more easily. Training was conducted for 40 epochs, including a 10-epoch warm-up phase, with a batch size of 64.

Pseudo-labels for the target domain were initialized from a source model and updated iteratively during training. To further refine performance, the activation threshold was set to $\beta = 0.05$, and the weights for all loss terms—including source domain classification loss, target domain classification loss, distillation loss, and progressively focused loss—were fixed at 1.0. All experiments were conducted on Nvidia V100 GPUs. The code and pretrained models will be made publicly available.

\section{Case Study of Matching Different Scales}

To further demonstrate the effectiveness of PCaM, we include a detailed case study showcasing examples of matching results at different scales in the source and target domains. This addresses the reviewers' comments on the importance of real-world examples.

Figure~\ref{fig:appendix_pcam_cs} illustrates the case study, highlighting how PCaM performs robust matching across varying scales, which is a critical challenge in domain adaptation tasks. The visualization provides insights into how our method aligns features between domains, even under significant scale variations.

\begin{figure}[h]
  \centering
  \includegraphics[width=0.55\linewidth]{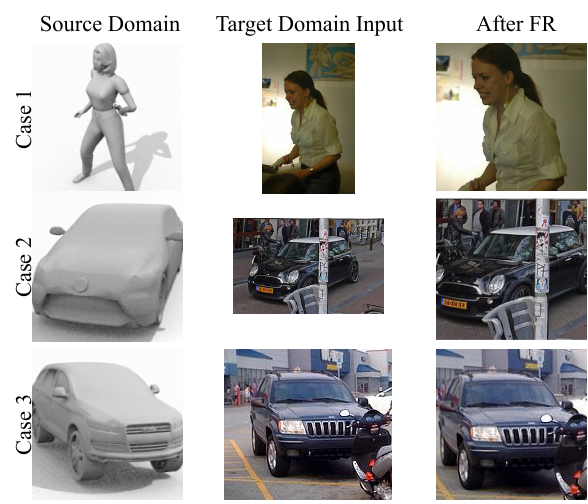}
  \caption{Case studies of matching different scales in the source and target domains.}
  \label{fig:appendix_pcam_cs}
\end{figure}

\section{Additional Ablation Studies}
\begin{table*}[t]
  \centering
  \caption{Performance comparison with different weights for $\mathcal{L}_\text{PF}$ on VisDA-2017 and DomainNet datasets. Mean and variance are reported across five runs.}
  \label{tab:appendix_loss_balance}
  \begin{tabular}{l|cccc}
    \toprule
    Dataset    & 0.2 $\mathcal{L}_\text{PF}$ & 0.5 $\mathcal{L}_\text{PF}$ & 1.0 $\mathcal{L}_\text{PF}$ & 2.0 $\mathcal{L}_\text{PF}$ \\ \midrule
    VisDA-2017 & {90.8}($\pm$0.08)    & \textbf{91.4}($\pm$0.06)    & 90.6($\pm$0.08)             & 90.7($\pm$0.09)             \\
    DomainNet  & 46.9($\pm$0.11)             & 47.0($\pm$0.10)             & \textbf{47.2}($\pm$0.12)    & 47.1($\pm$0.13)             \\ \bottomrule
  \end{tabular}
\end{table*}

In response to the request for additional ablation studies, we extended our analysis to evaluate the transferability of PCaM under various challenging scenarios, particularly focusing on low-resolution inputs. This study provides deeper insights into the robustness of our method. Table~\ref{tab:appendix_ablation_resolution} presents the results of these experiments on the VisDA-2017 dataset, comparing PCaM with the baseline method, CDTrans, across resolutions of 256$\times$256, 128$\times$128, and 64$\times$64.

As shown in the table, PCaM consistently outperforms CDTrans across all tested resolutions, demonstrating superior robustness to resolution changes. These findings further validate the effectiveness of PCaM in handling challenging domain adaptation tasks.

\begin{table}[h]
  \centering
  \caption{Performance comparison of PCaM and CDTrans under different input resolutions on the VisDA-2017 dataset.}
  \label{tab:appendix_ablation_resolution}
  \begin{tabular}{@{}c|ccc@{}}
    \toprule
    \textbf{Resolution} & \textbf{256$\times$256} & \textbf{128$\times$128} & \textbf{64$\times$64} \\ \midrule
    CDTrans             & 88.4             & 87.5             & 86.6           \\
    PCaM                & \textbf{91.4}    & \textbf{90.2}    & \textbf{89.6}  \\ \bottomrule
  \end{tabular}
\end{table}

\section{Balancing Contributions from Different Losses}

To address the reviewer's question about balancing the contributions of different losses, we conducted an analysis focusing on the integration of our additional loss, $\mathcal{L}_\text{PF}$, into the baseline model (CDTrans). Rather than adjusting all loss ratios, we fine-tuned only the contribution of $\mathcal{L}_\text{PF}$. This approach allowed us to assess its impact on the model's overall performance.

Table~\ref{tab:appendix_loss_balance} presents the results of our experiments on the VisDA-2017 and DomainNet datasets. We evaluated multiple weight settings for $\mathcal{L}_\text{PF}$ (0.2, 0.5, 1.0, and 2.0) and ran all experiments five times to account for stability, reporting the mean and variance. The results show that incorporating $\mathcal{L}_\text{PF}$ enhances performance, with optimal results observed at specific weight settings depending on the dataset. For VisDA-2017, 0.2 and 0.5 yielded the best performance, while for DomainNet, 1.0 provided the highest accuracy.

These findings suggest that the contribution of $\mathcal{L}_\text{PF}$ can be effectively balanced to achieve significant improvements without compromising stability.

\end{document}